\documentclass[review,11pt]{ReportTemplate}
\usepackage{bm}
\usepackage{graphicx}
\usepackage{paralist,algorithmic,algorithm}
\usepackage{amsmath,mathrsfs,amssymb}
\usepackage{theorem}
\usepackage[american]{babel}

\newtheorem{theorem}{Theorem}
\newtheorem{definition}{Definition}
\newtheorem{lemma}{Lemma}

\newcommand{\expect}[2][]{
    {\mathbb{E}_{#1}[\kern-0.15em[ #2 ]\kern-0.14em]}
    }

\newenvironment{proof}
{\par\vspace{0pt}\noindent\textbf{Proof.}\
\enspace\ignorespaces\begin{allowdisplaybreaks}}
{\end{allowdisplaybreaks}\hspace{\stretch{1}}$\square$\par}

\newenvironment{proof2}
{\par\vspace{0pt}\noindent\textbf{Proof of Theorem 2.}\
\enspace\ignorespaces\begin{allowdisplaybreaks}}
{\end{allowdisplaybreaks}\hspace{\stretch{1}}$\square$\par}

\newcommand{\UBoolean}{\texttt{UBoolean}}

\begin{document}
\begin{frontmatter}
\title{A Lower Bound Analysis of Population-based Evolutionary Algorithms for Pseudo-Boolean Functions}

\author{Chao Qian$^{1,2}$}
\author{Yang Yu$^2$}
\author{Zhi-Hua Zhou$^2$\corref{cor1}}
\address{$^1$UBRI, School of Computer Science and Technology,\\ University of Science and Technology of China, Hefei 230027, China\\
$^2$National Key Laboratory for Novel Software Technology,\\ Nanjing University, Nanjing 210023, China} \cortext[cor1]{\small Corresponding author.
Email: zhouzh@nju.edu.cn}

\begin{abstract}
Evolutionary algorithms (EAs) are population-based general-purpose optimization algorithms, and have been successfully applied in various real-world optimization tasks. However, previous theoretical studies often employ EAs with only a parent or offspring population and focus on specific problems. Furthermore, they often only show upper bounds on the running time, while lower bounds are also necessary to get a complete understanding of an algorithm. In this paper, we analyze the running time of the ($\mu$+$\lambda$)-EA (a general population-based EA with mutation only) on the class of pseudo-Boolean functions with a unique global optimum. By applying the recently proposed switch analysis approach, we prove the lower bound $\Omega(n \ln n+ \mu + \lambda n\ln\ln n/ \ln n)$ for the first time. Particularly on the two widely-studied problems, OneMax and LeadingOnes, the derived lower bound discloses that the ($\mu$+$\lambda$)-EA will be strictly slower than the (1+1)-EA when the population size $\mu$ or $\lambda$ is above a moderate order. Our results imply that the increase of population size, while usually desired in practice, bears the risk of increasing the lower bound of the running time and thus should be carefully considered.
\end{abstract}

\end{frontmatter}

\section{Introduction}

Evolutionary algorithms (EAs)~\cite{back:96} are a kind of population-based heuristic optimization algorithm. They have been widely applied in industrial optimization problems. However, the theoretical analysis is difficult due to their complexity and randomness. In the recent decade, there has been a significant rise on the running time analysis (one essential theoretical aspect) of EAs~\cite{auger2011theory,neumann2010bioinspired}. For example, Droste et al.~\cite{droste2002analysis} proved that the expected running time of the (1+1)-EA on linear pseudo-Boolean functions is $\Theta(n \ln n)$; for the ($\mu$+1)-EA solving several artificially designed functions, a large parent population size $\mu$ was shown to be able to reduce the running time from exponential to polynomial~\cite{jansen2001utility,storch2008choice,witt2006runtime,witt2008population}; for the (1+$\lambda$)-EA solving linear functions, the expected running time was proved to be $O(n \ln n + \lambda n)$~\cite{doerr2015optimizing}, and a tighter bound up to lower order terms was derived on the specific linear function OneMax~\cite{giessen2015population}.

Previous running time analyses often consider EAs with only a parent or offspring population, which do not fully reflect the population-based nature of real EAs. When involving both parent and offspring populations, the running time analysis gets more complex, and only a few results have been reported on the ($\lambda$+$\lambda$)-EA (i.e., a specific version of the ($\mu$+$\lambda$)-EA with $\mu=\lambda$), which maintains $\lambda$ solutions and generates $\lambda$ offspring solutions by only mutation in each iteration. He and Yao~\cite{he2002individual} compared the expected running time of the (1+1)-EA and the ($\lambda$+$\lambda$)-EA on two specific artificial problems, and proved that the introduction of a population can reduce the running time exponentially. On the contrary side, Chen et al.~\cite{chen2011large} found that a large population size is harmful for the ($\lambda$+$\lambda$)-EA solving the TrapZeros problem. Chen et al.~\cite{chen2009new} also proved that the expected running time of the ($\lambda$+$\lambda$)-EA on the OneMax and LeadingOnes problems is $O(\lambda n\ln \ln n+n\ln n)$ and $O(\lambda n \ln n+n^2)$, respectively. Later, a low selection pressure was shown to be better for the ($\lambda$+$\lambda$)-EA solving a wide gap problem~\cite{chen2010choosing}, and a proper mutation-selection balance was proved to be necessary for the effectiveness of the ($\lambda$+$\lambda$)-EA solving the SelPres problem~\cite{lehre2012impact}.

The above-mentioned studies on the ($\lambda$+$\lambda$)-EA usually focus on specific test functions, while EAs are general purpose optimization algorithms and can be applied to all optimization problems where solutions can be represented and evaluated. Thus, it is necessary to analyze EAs over large problem classes. Meanwhile, most previous running time analyses on population-based EAs only show upper bounds. Although upper bounds are appealing for revealing the ability of an algorithm, lower bounds which reveal the limitation are also necessary for a complete understanding of the algorithm.

In this paper, we analyze the running time of the ($\mu$+$\lambda$)-EA solving the class of \emph{pseudo-Boolean functions with a unique global optimum}, named $\UBoolean$, which covers many P and NP-hard combinatorial problems. By applying the recently proposed approach switch analysis~\cite{yu2014switch}, we prove that the expected running time is lower bounded by $\Omega(n \ln n+\mu+\lambda n\ln\ln n/ \ln n)$. Particularly, when applying this lower bound to the two specific problems, OneMax and LeadingOnes, we can have a more complete understanding of the impact of the offspring population size $\lambda$. It was known that the ($\mu$+$\lambda$)-EA is always not asymptotically faster than the (1+1)-EA on these two problems~\cite{lehre2010black,sudholt2011general}. But it was left open that what is the range of $\lambda$ where the ($\mu$+$\lambda$)-EA is asymptotically worse than the (1+1)-EA. Note that the expected running time of the (1+1)-EA on OneMax and LeadingOnes is $\Theta(n \ln n)$ and $\Theta(n^2)$, respectively~\cite{droste2002analysis}. By comparing them with our derived lower bound, we easily get that the ($\mu$+$\lambda$)-EA is strictly asymptotically slower than the (1+1)-EA when $\lambda \in \omega(\frac{(\ln n)^2}{\ln\ln n})$ on OneMax and $\lambda\in \omega(\frac{n \ln n}{\ln \ln n})$ on LeadingOnes. For the parent population size $\mu$, we easily get obvious ranges $\omega(n \ln n)$ and $\omega(n^2)$ for the ($\mu$+$\lambda$)-EA being asymptotically worse on OneMax and LeadingOnes, respectively.


The rest of this paper is organized as follows. Section 2 introduces some preliminaries. Section 3 introduces the employed analysis approach. The running time analysis of the ($\mu$+$\lambda$)-EA on $\UBoolean$ is presented in Section 4. Section 5 concludes the paper.

\section{Preliminaries}

In this section, we first introduce the ($\mu$+$\lambda$)-EA and the pseudo-Boolean problem class studied in this paper, respectively, then describe how to model EAs as Markov chains.

\subsection{($\mu$+$\lambda$)-EA}

Evolutionary algorithms (EAs)~\cite{back:96} are used as general heuristic randomized optimization approaches. Starting from an initial set of solutions (called a population), EAs try to improve the population by a cycle of three stages: reproducing new solutions from the current population, evaluating the newly generated solutions, and updating the population by removing bad solutions. The ($\mu$+$\lambda$)-EA as described in Algorithm~\ref{(N+N)-EA} is a general population-based EA with mutation only for optimizing pseudo-Boolean problems over $\{0,1\}^n$. It maintains $\mu$ solutions. In each iteration, one solution selected from the current population is used to generate an offspring solution by bit-wise mutation (i.e., line~5); this process is repeated independently for $\lambda$ times; then $\mu$ solutions out of the parent and offspring solutions are selected to be the next population. Note that the selection strategies for reproducing new solutions and updating the population can be arbitrary. Thus, the considered ($\mu$+$\lambda$)-EA is quite general, and covers most population-based EAs with mutation only in previous theoretical analyses, e.g.,~\cite{chen2011large,he2002individual,lehre2012impact}.

\begin{algorithm}[t]\caption{($\mu$+$\lambda$)-EA}
    Given solution length $n$ and objective function $f$, let every population, denoted by variable $\xi$, contain $\mu$ solutions. It consists of the following steps:\label{(N+N)-EA}
    \begin{algorithmic}[1]
    \STATE let $t\leftarrow 0$, and $\xi_0 \leftarrow \mu$ solutions uniformly and randomly selected from $\{0,1\}^n$.
    \STATE \textbf{repeat until} some criterion is met
    \STATE \quad \textbf{for} $i=1$ to $\lambda$
    \STATE \qquad select a solution $s$ from $\xi_t$ according to some selection mechanism.
    \STATE \qquad create $s'_i$ by flipping each bit of $s$ with probability $1/n$.
    \STATE \quad \textbf{end for}
    \STATE \quad $\xi_{t+1} :=$ select $\mu$ solutions from $\xi_t \cup \{s'_1,\ldots,s'_{\lambda}\}$ according to some strategy.
    \STATE \quad let $t\leftarrow t+1$.
    \end{algorithmic}
\end{algorithm}

The running time of EAs is usually defined as the number of fitness evaluations until an optimal solution is found for the first time, since the fitness evaluation is the computational process with the highest cost of the algorithm \cite{he2001drift,yu2008new}. Note that running time analysis has been a leading theoretical aspect for randomized search heuristics~\cite{auger2011theory,neumann2010bioinspired}.

\subsection{Pseudo-Boolean Function Problems}

The pseudo-Boolean function class is a large function class which only requires the solution space to be $\{0,1\}^n$ and the objective space to be $\mathbb{R}$. It covers many typical P and NP-hard combinatorial problems such as minimum spanning tree and minimum set cover. We consider a subclass named $\UBoolean$ as shown in Definition \ref{def_UBoolean}, in which every function has a unique global optimum. Note that maximization is considered since minimizing $f$ is equivalent to maximizing $-f$. For any function in $\UBoolean$, we assume without loss of generality that the optimal solution is $11\ldots1$ (briefly denoted as $1^n$). This is because EAs treat the bits 0 and 1 symmetrically, and thus the 0-bits in an optimal solution can be interpreted as 1-bits without affecting the behavior of EAs. The expected running time of unbiased black-box algorithms and mutation-based EAs on $\UBoolean$ has been proved to be $\Omega(n \ln n)$~\cite{lehre2010black,sudholt2011general}.

\begin{definition}[$\UBoolean$]\label{def_UBoolean}
    A function $f:\{0,1\}^n \rightarrow \mathbb{R}$ in $\UBoolean$ satisfies that
    $$
        \exists s \in \{0,1\}^n, \forall s' \in \{0,1\}^n-\{s\}, f(s')<f(s).
    $$
\end{definition}

Diverse pseudo-Boolean problems in $\UBoolean$ have been used for analyzing the running time of EAs, and then to disclose properties of EAs. Here, we introduce the LeadingOnes problem, which will be used in this paper. As presented in Definition~\ref{def_leadingones}, it is to maximize the number of consecutive 1-bits starting from the left. It has been proved that the expected running time of the (1+1)-EA on LeadingOnes is $\Theta(n^2)$~\cite{droste2002analysis}.

\begin{definition}[LeadingOnes]\label{def_leadingones}
    LeadingOnes Problem of size $n$ is to find an $n$ bits binary string $s^*$ such that, letting $s_j$ be the $j$-th bit of a solution $s \in \{0,1\}^n$,
    $$
        s^*=\mathop{\arg\max}\nolimits_{s \in \{0,1\}^n} \left(f(s)=\sum\nolimits^{n}_{i=1} \prod\nolimits^{i}_{j=1} s_j\right).
    $$
\end{definition}

\subsection{Markov Chain Modeling}

EAs can be modeled and analyzed as Markov chains, e.g., in~\cite{he2001drift,yu2008new}. Let $\mathcal{X}$ be the population space and $\mathcal{X}^* \subseteq \mathcal{X}$ be the optimal population space. Note that an optimal population in $\mathcal{X}^*$ contains at least one optimal solution. Let $\xi_t \in \mathcal{X}$ be the population after $t$ generations. Then, an EA can be described as a random sequence $\{\xi_0,\xi_1,\xi_2,\ldots\}$. Since $\xi_{i+1}$ can often be decided from $\xi_{i}$ and the reproduction operator of the EA (i.e., $P(\xi_{i+1}\mid\xi_{i},\xi_{i-1},\ldots, \xi_0) = P(\xi_{i+1}\mid\xi_{i})$), the random sequence forms a Markov chain $\{\xi_t\}_{t=0}^{+\infty}$ with state space $\mathcal{X}$, denoted as ``$\xi\in \mathcal{X}$'' for simplicity. Note that all sets considered in this paper are multisets, e.g., a population can contain several copies of the same solution.

The goal of EAs is to reach the optimal space $\mathcal{X}^*$ from an initial population $\xi_0$. Given a Markov chain $\xi\in \mathcal{X}$ modeling an EA and $t_0 \geq 0$, we define $\tau$ as a random variable such that $\tau=\min\{t \geq 0 \mid \xi_{t_0+t} \in \mathcal{X}^*\}$. That is, $\tau$ is the number of steps needed to reach the optimal space for the first time when starting from time $t_0$. The mathematical expectation of $\tau$, $\expect{\tau \mid \xi_{t_0}=x}=\sum\nolimits^{\infty}_{i=0} i P(\tau=i)$, is called the \emph{conditional first hitting time} (CFHT) of the chain staring from $\xi_{t_0}=x$. If $\xi_{t_0}$ is drawn from a distribution $\pi_{t_0}$, the expectation of the CFHT over $\pi_{t_0}$, $\expect{\tau \mid \xi_{t_0}\sim \pi_{t_0} } = \sum\nolimits_{x\in \mathcal{X}} \pi_{t_0}(x)\expect{\tau \mid \xi_{t_0}=x}$, is called the \emph{distribution-CFHT} (DCFHT) of the chain from $\xi_{t_0} \sim \pi_{t_0}$. Since the running time of EAs is counted by the number of fitness evaluations, the cost of initialization and each generation should be considered. For example, the expected running time of the ($\mu$+$\lambda$)-EA is $\mu+\lambda \cdot \expect{\tau \mid \xi_0 \sim \pi_0}$.

A Markov chain $\xi \in \mathcal{X}$ is said to be absorbing, if $\forall t \geq 0: P(\xi_{t+1} \in \mathcal{X}^* \mid \xi_t \in \mathcal{X}^*) = 1$. Note that all Markov chains modeling EAs can be transformed to be absorbing by making it unchanged once an optimal state has been found. This transformation obviously does not affect its first hitting time.

\section{The Switch Analysis Approach}

To derive running time bounds of the ($\mu$+$\lambda$)-EA on $\UBoolean$, we first model the EA process as a Markov chain, and then apply the switch analysis approach.

Switch analysis~\cite{yu2014switch,yu2015switch} as presented in Theorem~\ref{them_main} is a recently proposed approach that compares the DCFHT of two Markov chains. Since the state spaces of the two chains may be different, an aligned mapping function $\phi:\mathcal{X}\to \mathcal{Y}$ as shown in Definition~\ref{def:optmap} is employed. Note that $\phi^{-1}(y)=\{x\in\mathcal{X}\mid \phi(x)=y\}$. Using switch analysis to derive running time bounds of a given chain $\xi \in \mathcal{X}$ (i.e., modeling a given EA running on a given problem), one needs to
\begin{enumerate}
  \item construct a reference chain $\xi' \in \mathcal{Y}$ for comparison and design an aligned mapping function $\phi$ from $\mathcal{X}$ to $\mathcal{Y}$;
  \item analyze their one-step transition probabilities, i.e., $P(\xi_{t+1} \mid \xi_{t})$ and $P(\xi'_{t+1} \mid \xi'_{t})$, the CFHT of the chain $\xi' \in \mathcal{Y}$, i.e., $\expect{\tau' \mid \xi'_{t}}$, and the state distribution of the chain $\xi \in \mathcal{X}$, i.e., $\pi_t$;
  \item examine Eq.~(\refeq{SA-condition}) to get the difference $\rho_t$ between each step of the two chains;
  \item sum up $\rho_t$ to get a running time gap $\rho$ of the two chains, and then bounds on $\expect{\tau \mid \xi_0}$ can be derived by combining $\expect{\tau' \mid \xi'_0}$ with $\rho$.
\end{enumerate}

\begin{definition}[Aligned Mapping~\cite{yu2014switch}]\label{def:optmap}
Given two spaces $\mathcal{X}$ and $\mathcal{Y}$ with target subspaces $\mathcal{X}^*$ and $\mathcal{Y}^*$, respectively, a function $\phi:\mathcal{X}\to \mathcal{Y}$ is called \\
(a) a \emph{left-aligned mapping} if $\ \ \forall x\in \mathcal{X}^*: \phi(x)\in \mathcal{Y}^*$;\\
(b) a \emph{right-aligned mapping} if $\ \ \forall x\in \mathcal{X}-\mathcal{X}^*:\phi(x)\notin \mathcal{Y}^*$;\\
(c) an \emph{optimal-aligned mapping} if it is both left-aligned and right-aligned.
\end{definition}

\begin{theorem}[Switch Analysis~\cite{yu2014switch}]\label{them_main}
        Given two absorbing Markov chains $\xi\in \mathcal{X}$ and $\xi'\in\mathcal{Y}$, let $\tau$ and $\tau'$ denote the hitting events of $\xi$ and $\xi'$, respectively, and let $\pi_t$ denote the distribution of $\xi_t$. Given a series of values $\{\rho_t\in\mathbb{R}\}_{t=0}^{+\infty}$ with $\rho= \sum_{t=0}^{+\infty} \rho_t$ and a right (or left)-aligned mapping $\phi: \mathcal{X} \rightarrow \mathcal{Y}$, if ~$\expect{\tau \mid \xi_{0} \sim \pi_0}$ is finite and
        \begin{equation}
        \begin{aligned}\label{SA-condition}
            & \forall t: \sum\limits_{x\in \mathcal{X}, y \in \mathcal{Y}} \pi_t(x) P(\xi_{t+1}\in\phi^{-1}(y) \mid \xi_t=x) \expect{\tau' \mid \xi'_{0} = y} \\
            & \leq(\text{or }\geq)  \sum\limits_{u,y\in \mathcal{Y}}  {\pi^\phi_t}(u) P(\xi'_{1}=y \mid \xi'_0=u) \expect{\tau' \mid \xi'_{1} =y}+\rho_t,
        \end{aligned}
        \end{equation}
         where $\pi^\phi_t(u)=\pi_t(\phi^{-1}(u))=\sum_{x\in \phi^{-1}(u)}\pi_t(x)$, we have
         $$
         \expect{\tau \mid \xi_{0} \sim \pi_0} \leq(\text{or }\geq) \expect{\tau' \mid \xi'_{0} \sim \pi^\phi_0} +\rho .
         $$
\end{theorem}

The idea of switch analysis is to obtain the difference $\rho$ on the DCFHT of two chains by summing up all the one-step differences $\rho_t$. Using Theorem~\ref{them_main} to compare two chains, we can waive the long-term behavior of one chain, since Eq.~\eqref{SA-condition} does not involve the term $\expect{\tau \mid \xi_{t}}$. Therefore, the theorem can simplify the analysis of an EA process by comparing it with an easy-to-analyze one.

\section{Running Time Analysis}

In this section, we prove a lower bound on the expected running time of the ($\mu$+$\lambda$)-EA solving $\UBoolean$, as shown in Theorem~\ref{N+N-EA_lower}. Our proof is accomplished by using switch analysis (i.e., Theorem~\ref{them_main}). The target EA process we are to analyze is the ($\mu$+$\lambda$)-EA running on any function in $\UBoolean$. The constructed reference process for comparison is the RLS$^{\neq}$ algorithm running on the LeadingOnes problem. RLS$^{\neq}$ is a modification of the randomized local search algorithm. It maintains only one solution $s$. In each iteration, a new solution $s'$ is generated by flipping a randomly chosen bit of $s$, and $s'$ is accepted only if $f(s')> f(s)$. That is, RLS$^{\neq}$ searches locally and only accepts a better offspring solution.

\begin{theorem}\label{N+N-EA_lower}
The expected running time of the ($\mu$+$\lambda$)-EA on $\UBoolean$ is $\Omega(n \ln n + \mu+ \lambda n \ln\ln n/\ln n)$, when $\mu$ and $\lambda$ are upper bounded by a polynomial in $n$.
\end{theorem}

We first give some lemmas that will be used in the proof of Theorem~\ref{N+N-EA_lower}. Lemma~\ref{lem_cfht} characterizes the one-step transition behavior of a Markov chain via CFHT. Lemma~\ref{RLS_LeadingOnes} gives the CFHT $\expect{\tau'\mid \xi'_t=y}$ of the reference chain $\xi'$ (i.e., RLS$^{\neq}$ running on LeadingOnes). In the following analysis, we will use $\mathbb{E}_{rls}(j)$ to denote $\expect{\tau'\mid \xi'_t=y}$ with $|y|_0=j$, i.e., $\mathbb{E}_{rls}(j)=nj$.

\begin{lemma}[\cite{freidlin1996markov}]\label{lem_cfht}
        Given a Markov chain $\xi \in \mathcal{X}$ and a target subspace $\mathcal{X}^*\subset \mathcal{X}$, we have, for CFHT,\;\; $\forall x \in \mathcal{X}^*: \expect{\tau \mid \xi_t=x}=0$,
        \begin{align*}
        &\forall x\notin \mathcal{X}^*:\expect{\tau \mid \xi_t=x}=1+\sum\nolimits_{x'\in \mathcal{X}} P(\xi_{t+1}=x' \mid \xi_t=x)\expect{\tau\mid \xi_{t+1}=x'}.
        \end{align*}
\end{lemma}

\begin{lemma}[\cite{yu2014switch}]\label{RLS_LeadingOnes}
For the chain $\xi' \in \mathcal{Y}$ modeling RLS$^{\neq}$ running on the LeadingOnes problem, the CFHT satisfies that $\forall y \in \mathcal{Y}=\{0,1\}^n: \expect{\tau'\mid \xi'_t=y}=n\cdot |y|_0$, where $|y|_0$ denotes the number of 0-bits of $y$.
\end{lemma}

\begin{lemma}\label{compro1}
For $m \geq i \geq 0$, $\sum^i_{k=0} \binom{m}{k}(\frac{1}{n})^k(1-\frac{1}{n})^{m-k}$ decreases with $m$.
\end{lemma}
\begin{proof}
Let $f(m)=\sum^i_{k=0} \binom{m}{k}(\frac{1}{n})^k(1-\frac{1}{n})^{m-k}$. The goal is to show that $f(m+1) \leq f(m)$ for $m \geq i$. Denote $X_1,\ldots,X_{m+1}$ as independent random variables, where $X_j$ satisfies that $P(X_j=1)=\frac{1}{n}$ and $P(X_j=0)=1-\frac{1}{n}$. Then we can express $f(m)$ and $f(m+1)$ as $f(m)=P(\sum\nolimits^m_{j=1} X_j \leq i)$ and $f(m+1)=P(\sum\nolimits^{m+1}_{j=1} X_j \leq i)$. Thus,
\begin{align*}
& f(m+1)=P(\sum\nolimits^m_{j=1} X_j < i)+ P(\sum\nolimits^m_{j=1} X_j = i)P(X_{m+1}=0)\\
& =P(\sum\nolimits^m_{j=1} X_j < i)+ P(\sum\nolimits^m_{j=1} X_j = i)(1-\frac{1}{n}) \leq f(m).
\end{align*}
\end{proof}

\begin{lemma}\label{compro2}
For $\lambda \leq n^c$ where $c$ is a positive constant, it holds that
$$
\sum^{n-1}_{i=0} \left(\sum^i_{k=0} \binom{n}{k} (\frac{1}{n})^k(1-\frac{1}{n})^{n-k}\right)^{\lambda} \geq n-\left\lceil \frac{e(c+1)\ln n}{\ln \ln n} \right\rceil.
$$
\end{lemma}
\begin{proof}
Let $m=\lceil \frac{e(c+1)\ln n}{\ln \ln n} \rceil$. Denote $X_1,...,X_n$ as independent random variables, where $P(X_j=1)=\frac{1}{n}$ and $P(X_j=0)=1-\frac{1}{n}$. Let $X=\sum^n_{j=1} X_j$, then its expectation $\expect{X}=1$. We thus have
\begin{equation}
\begin{aligned}\label{Chernoff}
\forall i \geq m, &\sum^n_{k=i} \binom{n}{k} (\frac{1}{n})^k(1-\frac{1}{n})^{n-k}=P(X \geq i) \leq e^{(i-1)}/i^{i},
\end{aligned}
\end{equation}
where the inequality is by Chernoff bound.
Then, we have
\begin{align*}
 & \sum^{n-1}_{i=0} (\sum^i_{k=0} \binom{n}{k} (\frac{1}{n})^k(1-\frac{1}{n})^{n-k})^{\lambda}\geq \sum^{n-1}_{i=m-1} (\sum^i_{k=0} \binom{n}{k} (\frac{1}{n})^k(1-\frac{1}{n})^{n-k})^{\lambda}\\
 &= \sum^{n-1}_{i=m-1} (1-\sum^n_{k=i+1} \binom{n}{k} (\frac{1}{n})^k(1-\frac{1}{n})^{n-k})^{\lambda}\\
 & \geq \sum^{n-1}_{i=m-1} (1-e^{i}/(i+1)^{(i+1)})^{\lambda}\geq \sum^{n-1}_{i=m-1} (1-e^{(m-1)}/m^m)^{\lambda},
\end{align*}
where the second inequality is by Eq.~(\refeq{Chernoff}), and the last inequality can be easily derived because $e^i/(i+1)^{(i+1)}$ decreases with $i$ when $i \geq m-1$.

Then, we evaluate $e^{m-1}/m^{m}$ by taking logarithm to its reciprocal.
\begin{align*}
&\ln(m^m/e^{m-1})=m(\ln m -1)+1\\
&\geq \frac{e(c+1)\ln n}{\ln \ln n}(1+\ln (c+1)+\ln\ln n - \ln\ln\ln n-1)+1\\
& \geq \frac{e(c+1)\ln n}{\ln \ln n} \frac{1}{e} \ln\ln n=(c+1)\ln n\geq \ln \lambda n. \quad (\text{by $\lambda \leq n^c$})
\end{align*}
This implies that $e^{m-1}/m^{m} \leq \frac{1}{\lambda n}$. Thus, we have
\begin{align*}
& \sum^{n-1}_{i=0} (\sum^i_{k=0} \binom{n}{k} (\frac{1}{n})^k(1-\frac{1}{n})^{n-k})^{\lambda} \geq \sum^{n-1}_{i=m-1} (1-\frac{1}{\lambda n})^{\lambda}\\
&\geq \sum^{n-1}_{i=m-1} (1-\frac{1}{n}) \geq n-m =n-\lceil \frac{e(c+1)\ln n}{\ln \ln n}\rceil,
\end{align*}
where the second inequality is by $\forall 0 \leq xy \leq 1, y \geq 1:(1-x)^y \geq  1-xy$.
\end{proof}

\begin{lemma}[\cite{flum2006}]\label{compro3} Let $H(\epsilon)=-\epsilon \log \epsilon -(1-\epsilon) \log(1-\epsilon)$. It holds that
$$
\forall n \geq 1, 0<\epsilon<\frac{1}{2}: \sum^{\lfloor \epsilon n \rfloor}_{k=0} \binom{n}{k} \leq 2^{H(\epsilon)n}.
$$
\end{lemma}

\begin{proof2}
We use switch analysis (i.e., Theorem~\ref{them_main}) to prove it. Let $\xi \in \mathcal{X}$ model the analyzed EA process (i.e., the ($\mu$+$\lambda$)-EA running on any function in $\UBoolean$). We use RLS$^{\neq}$ running on the LeadingOnes problem as the reference process modeled by $\xi' \in \mathcal{Y}$. Then, $\mathcal{Y}=\{0,1\}^n$, $\mathcal{X}=\{\{y_1,y_2,\ldots,y_{\mu}\}\mid y_i \in \{0,1\}^n\}$, $\mathcal{Y}^*=\{1^n\}$ and $\mathcal{X}^*=\{x \in \mathcal{X} \mid \max_{y \in x}|y|_1=n\}$, where $|y|_1$ denotes the number of 1-bits of a solution $y \in \{0,1\}^n$. We construct a mapping $\phi: \mathcal{X} \rightarrow \mathcal{Y}$ as that $\forall x \in \mathcal{X}: \phi(x)= \arg \max_{y \in x} |y|_1$. It is easy to see that the mapping is an optimality-aligned mapping, because $\phi(x) \in \mathcal{Y}^*$ iff $x \in \mathcal{X}^*$.

We investigate the condition Eq.~(\ref{SA-condition}) of switch analysis. For any $x \notin \mathcal{X}^*$, suppose that $\min \{|y|_0 \mid y \in x\}=j>0$. Then, $|\phi(x)|_0=j$. By Lemma~\ref{lem_cfht} and~\ref{RLS_LeadingOnes}, we have
\begin{equation}
\begin{aligned}\label{onestep-nonoptimal9}
&\sum\nolimits_{y\in \mathcal{Y}}P(\xi'_1=y \mid \xi'_0=\phi(x)) \expect{\tau' \mid\xi'_1=y}=\mathbb{E}_{rls}(j)-1=nj-1.
\end{aligned}
\end{equation}
For the reproduction of the ($\mu$+$\lambda$)-EA (i.e., the chain $\xi \in \mathcal{X}$) on the population $x$, assume that the $\lambda$ selected solutions from $x$ for reproduction have the number of 0-bits $j_1,j_2,...,j_{\lambda}$, respectively, where $j \leq j_1 \leq j_2 \leq ... \leq j_{\lambda} \leq n$. If there are at most $i\;(0 \leq i \leq j_1)$ number of 0-bits mutating to 1-bits for each selected solution and there exists at least one selected solution which flips exactly $i$ number of 0-bits, which happens with probability $\prod^{\lambda}_{p=1}(\sum^i_{k=0} \binom{j_p}{k} (\frac{1}{n})^k(1-\frac{1}{n})^{j_p-k})-\prod^{\lambda}_{p=1}(\sum^{i-1}_{k=0} \binom{j_p}{k} (\frac{1}{n})^k(1-\frac{1}{n})^{j_p-k})$ (denoted by $p(i)$), the next population $x'$ satisfies that $|\phi(x')|_0 \geq j_1-i$. Furthermore, $\mathbb{E}_{rls}(i)=ni$ increases with $i$. Thus, we have
\begin{equation}
\begin{aligned}\label{onestep-nonoptimal10}
&\sum\limits_{y \in \mathcal{Y}} P(\xi_{t+1} \in \phi^{-1}(y) \mid \xi_t=x) \expect{\tau' \mid \xi'_0 = y} \geq \sum^{j_1}_{i=0} p(i) \cdot \mathbb{E}_{rls}(j_1-i)\\ &
\geq \sum^{j}_{i=0} p(i) \cdot \mathbb{E}_{rls}(j-i)=n\sum^{j-1}_{i=0} (\prod^{\lambda}_{p=1}(\sum^i_{k=0} \binom{j_p}{k} (\frac{1}{n})^k(1-\frac{1}{n})^{j_p-k})).
\end{aligned}
\end{equation}
By comparing Eq.~(\refeq{onestep-nonoptimal9}) with Eq.~(\refeq{onestep-nonoptimal10}), we have $\forall x \notin \mathcal{X}^*$,
\begin{align*}
&\sum\limits_{y \in \mathcal{Y}} P(\xi_{t+1}\!\in\! \phi^{-1}(y) \mid \xi_t\!=\!x) \expect{\tau' | \xi'_0 \!=\! y}- \sum\limits_{y\in \mathcal{Y}}P(\xi'_1\!=\!y \mid \xi'_0\!=\!\phi(x)) \expect{\tau' |\xi'_1\!=\!y}\\
& \geq n(\sum^{j-1}_{i=0}(\prod^{\lambda}_{p=1}(\sum^i_{k=0} \binom{j_p}{k} (\frac{1}{n})^k(1-\frac{1}{n})^{j_p-k}))-j)+1\\
&\geq n(\sum^{j-1}_{i=0}(\sum^i_{k=0} \binom{n}{k} (\frac{1}{n})^k(1-\frac{1}{n})^{n-k})^{\lambda}-j)+1\\
&\geq n(\sum^{n-1}_{i=0} (\sum^i_{k=0} \binom{n}{k} (\frac{1}{n})^k(1-\frac{1}{n})^{n-k})^{\lambda}-n)+1,
\end{align*}
where the 2nd `$\geq$' is because from Lemma \ref{compro1}, $\sum^i_{k=0} \binom{m}{k} (\frac{1}{n})^k(1-\frac{1}{n})^{m-k}$ reaches the minimum when $m=n$, and the last `$\geq$' is by $(\sum^i_{k=0} \binom{n}{k} (\frac{1}{n})^k(1-\frac{1}{n})^{n-k})^{\lambda} \leq 1$.\\
When $x \in \mathcal{X}^*$, both Eq.~(\ref{onestep-nonoptimal9}) and Eq.~(\ref{onestep-nonoptimal10}) equal 0, because both chains are absorbing and the mapping $\phi$ is optimality-aligned. Thus, Eq.~(\ref{SA-condition}) in Theorem~\ref{them_main} holds with $\rho_t=(n(\sum^{n-1}_{i=0} (\sum^i_{k=0} \binom{n}{k} (\frac{1}{n})^k(1-\frac{1}{n})^{n-k})^{\lambda}-n)+1)(1-\pi_t(\mathcal{X}^*))$. By switch analysis,
\begin{align*}
\expect{\tau | \xi_{0} \sim \pi_0} \geq& \expect{\tau' | \xi'_{0} \sim \pi^{\phi}_0}\\
&+(n(\sum^{n-1}_{i=0} (\sum^i_{k=0} \binom{n}{k} (\frac{1}{n})^k(1-\frac{1}{n})^{n-k})^{\lambda}-n)+1) \sum^{+\infty}_{t=0} (1-\pi_t(\mathcal{X}^*)).
\end{align*}
Since $\sum^{+\infty}_{t=0} (1-\pi_t(\mathcal{X}^*))=\expect{\tau | \xi_{0} \sim \pi_0}$, we have
\begin{equation}
\begin{aligned}\label{lower-bound}
&\expect{\tau | \xi_{0} \sim \pi_0} \geq \frac{\expect{\tau' | \xi'_{0} \sim \pi^{\phi}_0}}{n(n-\sum\limits^{n-1}_{i=0} (\sum\limits^i_{k=0} \binom{n}{k} (\frac{1}{n})^k(1-\frac{1}{n})^{n-k})^{\lambda})} \geq \frac{\expect{\tau' | \xi'_{0} \sim \pi^{\phi}_0}}{n\lceil \frac{e(c+1)\ln n}{\ln \ln n} \rceil},
\end{aligned}
\end{equation}
where the last inequality is by Lemma~\ref{compro2}, since $\lambda \leq n^c$ for some constant $c$.

We then investigate $\expect{\tau' | \xi'_{0} \sim \pi^{\phi}_0}$. Since each of the $\mu$ solutions in the initial population is selected uniformly and randomly from $\{0,1\}^n$, we have
\begin{align*}\label{initial_dist}
\forall 0 \leq j \leq n:\;& \pi^{\phi}_0(\{y \in \mathcal{Y}\mid |y|_0=j\})=\pi_0(\{x \in \mathcal{X} \mid \min_{y \in x} |y|_0=j\})\\
&=\frac{(\sum^n_{k=j} \binom{n}{k})^{\mu}-(\sum^n_{k=j+1} \binom{n}{k})^{\mu}}{2^{n\mu}},
\end{align*}
where $\sum^n_{k=j} \binom{n}{k}$ is the number of solutions with not less than $j$ number of 0-bits. Then,
\begin{align*}
&\expect{\tau' | \xi'_{0} \sim \pi^{\phi}_0}=\sum\nolimits^n_{j=0} \pi^{\phi}_0(\{y \in \mathcal{Y}\mid |y|_0=j\}) \mathbb{E}_{rls}(j)\\
&=\frac{1}{2^{n\mu}} \sum^n_{j=1}((\sum^n_{k=j} \binom{n}{k})^{\mu}-(\sum^n_{k=j+1} \binom{n}{k})^{\mu})nj =\frac{n}{2^{n\mu}} \sum^n_{j=1}(\sum^n_{k=j} \binom{n}{k})^{\mu}\\
&> n\sum^{\lfloor \frac{n}{4} \rfloor+1}_{j=1} (\sum^n_{k=j} \binom{n}{k}/2^n)^{\mu}> \frac{n^2}{4}(\sum^n_{k=\lfloor \frac{n}{4} \rfloor+1} \binom{n}{k}/2^n)^{\mu}=\frac{n^2}{4}(1-\sum^{\lfloor \frac{n}{4} \rfloor}_{k=0} \binom{n}{k}/2^n)^{\mu}\\
& \geq \frac{n^2}{4}(1-2^{H(\frac{1}{4})n-n})^{\mu} \geq \frac{n^2}{4} e^{-\frac{\mu}{2^{(1-H(\frac{1}{4}))n}-1}}>\frac{n^2}{4}e^{-\frac{\mu}{1.13^n-1}},
\end{align*}
where the third inequality is by Lemma \ref{compro3}, the fourth inequality is by $\forall 0< x<1: (1-x)^y \geq e^{-\frac{xy}{1-x}}$, and the last inequality is by $2^{1-H(\frac{1}{4})} > 1.13$.

Applying the above lower bound on $\expect{\tau' | \xi'_{0} \sim \pi^{\phi}_0}$ to Eq.~(\refeq{lower-bound}), we get, noting that $\mu$ is upper bounded by a polynomial in $n$,
$$
\expect{\tau | \xi_{0} \sim \pi_0} \geq \frac{n}{4\lceil \frac{e(c+1)\ln n}{\ln \ln n} \rceil}e^{-\frac{\mu}{1.13^n-1}},\quad \text{i.e.,}\quad \Omega(\frac{n\ln\ln n}{\ln n}).
$$
Considering the $\mu$ number of fitness evaluations for the initial population and the $\lambda$ number of fitness evaluations in each generation, the expected running time of the ($\mu$+$\lambda$)-EA on $\UBoolean$ is lower bounded by $\Omega(\mu+\frac{\lambda n\ln\ln n}{\ln n})$. Because the ($\mu$+$\lambda$)-EA belongs to mutation-based EAs, we can also directly use the general lower bound $\Omega(n\ln n)$~\cite{sudholt2011general}. Thus, the theorem holds.
\end{proof2}


\section{Conclusion}

This paper analyzes the expected running time of the ($\mu$+$\lambda$)-EA for solving a general problem class consisting of pseudo-Boolean functions with a unique global optimum. We derive the lower bound $\Omega(n \ln n + \mu + \lambda n \ln\ln n/\ln n)$ by applying the recently proposed approach switch analysis. The results partially complete the running time comparison between the ($\mu$+$\lambda$)-EA and the (1+1)-EA on the two well-studied pseudo-Boolean problems, OneMax and LeadingOnes. We can now conclude that when $\mu$ or $\lambda$ is slightly large, the ($\mu$+$\lambda$)-EA has a worse expected running time. The investigated ($\mu$+$\lambda$)-EA only uses mutation, while crossover is a characterizing feature of EAs. Therefore, we will try to analyze the running time of population-based EAs with crossover operators in the future.

%

\bibliography{population}\bibliographystyle{alpha}
\end{document}